\documentclass[11pt]{article}
\usepackage{amsmath}
\usepackage{amssymb}
\usepackage{hyperref}
\usepackage{amsthm}
\usepackage{natbib}
\usepackage[utf8]{inputenc}
\usepackage{bbm}
\usepackage{tikz}
\usepackage{amsmath}
\usepackage{multirow}
\usepackage{bm}
\usepackage{stmaryrd}
\DeclareMathOperator*{\bigtimes}{\vartimes}
\usetikzlibrary{arrows}
\usepackage{rotating}
\newcommand{\indep}{\rotatebox[origin=c]{90}{$\models$}}
\newcommand{\Pp}{\operatorname{P}}

\newcommand{\bx}{\mathbf{x}}

\newcommand{\bX}{\mathbf{X}}
\newcommand{\bchi}{\bm{\mathcal{X}}}
\newcommand{\nchi}{\mathcal{X}}

\newcommand{\sign}{\operatorname{sign}}

\newtheorem{definition}{Definition}

\newtheorem{example}{Example}
\newtheorem{lemma}{Lemma}
\newtheorem{remark}{Remark}
\newtheorem{theorem}{Theorem}
\newtheorem{corollary}{COrollary}
\date{}
\begin{document}

\title{Markov Property in Generative Classifiers}
\author{Gherardo Varando \footnote{gherardo.varando@math.ku.dk}
 \and  Concha Bielza \and
   Pedro Larra\~naga \and
   Eva Riccomagno        }

\maketitle

\begin{abstract}
We show that, for  generative classifiers, conditional independence corresponds
to linear constraints for the induced discrimination functions. Discrimination
functions of undirected Markov network classifiers can thus be characterized by
sets of linear constraints.  These constraints are represented by a
second order finite difference operator over functions of categorical variables.
As an application we study the expressive power of generative classifiers under
the undirected Markov property and we  present a general method to combine
discriminative and generative classifiers. 
\end{abstract}

\section{Introduction}

Generative classifiers are a wide class of machine learning models that consist
in estimating the joint probability distributions over the predictor and class
variables. From the estimated distribution a decision can be made over the class
variable given the values of the predictors.

Algebraic and geometric methods can be  valuable tools in dealing with discrete
probabilities as graphical models~\citep{GARCIA2005, Settimi1998}, contingency
tables and exponential families~\citep{Diaconis95, fienberg70}. \cite{varando15}
have studied the decision functions induced by a large class of generative
classifiers based on Bayesian networks, extending the results of
\cite{Minsky61stepstoward, peot} and \cite{Jaeger03}. \cite{ling2003} have
described the complexity of Bayesian network classifiers linking the graph
structure with the maximum order of the XORs that are representable by the
corresponding classifier.

In this article we develop a framework to study generative binary classifiers,
over categorical predictors, under conditional independences.  In
Section~\ref{sec:generative} we present generative classifiers together with
some basic notation and definitions. We study the implications of conditional
independence statements for the induced decision of generative classifiers in
Section~\ref{sec:markov}. In particular, in Section~\ref{sec:diffoper}, we
define a difference operator acting on discrimination functions and in
Section~\ref{sec:condindep} we prove that conditional independence statements,
for generative classifiers, are equivalent to a particular second order
difference being equal to zero for the induced discrimination function. In
Section~\ref{sec:markovclassifier} we extend the equivalence to generative
classifiers under undirected Markov assumptions.  Section~\ref{sec:complexity}
considers the  complexity of generative classifiers extending the results of
\cite{ling2003} and \cite{varando15} to Markov network.  In
Section~\ref{sec:learning} we suggest a simple way to combine generative and
discriminative approaches finding the maximum-likelihood estimation of the
parameters for generative classifiers with a  given discrimination function and
under undirected Markov assumptions. In Section~\ref{sec:conclusions} we
summarize the main result and address  some possible future lines of research.

\section{Generative Classifiers}
\label{sec:generative}

The classical problem of binary supervised classification in machine learning
amounts to \textit{learning} the relationship between some random variables,
called predictors or features, and a binary random variable called class.  Let
$C$ be the class variable assuming values in $\{-1,+1\}$ and $X_1,\ldots,X_n$ the 
$n$ predictor variables.

Let $\mathbf{X} = (X_1,\ldots,X_n)$ be the vector of predictors and
$\mathbf{x}=(x_1,\ldots,x_n)$ a vector of values of the features. Let $[n]$ be
the set of the first $n$ positive integer {numbers} $\{1,\ldots,n\}$.  Given $A
\subseteq [n]$, let $\mathbf{X}_A=(X_i)_{i\in A}$ indicate the vector of
{predictors indexed by elements } in $A$ and similarly
$\mathbf{x}_{A}=(x_i)_{i\in A}$. Conversely let $\bX_{-A}$ be the vector with
components not in $A$, $(X_i)_{i\in [n]\setminus A}$, and similarly for
$\bx_{-A}$.   {If $A=[n]$, $\bX_{-A}$ is the empty string}. With $|A|$ we will
denote the cardinality of  $A$ and with $-A = [n]-A$ the complementary
set.
For $i \in [n]$ let $X_i$ take values in a finite set, $\nchi_i$.  Let  $\bchi =
\bigtimes_{i \in [n]} \nchi_i$ denote the  sample space of the random
vector $\mathbf{X}$ and  $\bchi_A= \bigtimes_{i \in A} \nchi_i$ the {sample
space} of the subvector $(X_i)_{i\in A}$.  

If $\Pp$  is a joint probability, let $p$  denote  the
corresponding density, for example $p(\bx,c) = \Pp(\bX=\bx, C=c)$ is the joint
density of $(\bX,C)$ the random vector including the predictor and class
variables.  

Following classical literature on machine learning~\citep{ptpr} let a binary
classifier over predictors $\bX$ be a function $\phi: \bchi \to \{-1,+1\}$.

In this work we focus on generative classifiers where the classification
function is constructed from a joint probability distribution.
\begin{definition} \label{def:genclassifier}

A generative (probabilistic) binary classifier over predictor variables $\bX$ is
a strictly positive probability distribution $\Pp$ over $\boldsymbol{\nchi}
\times \{-1,+1\}$.  The induced binary classifier (or induced decision function)
of $\Pp$ is defined as the  most probable a posteriori
class,

$$\phi_{\Pp}(\bx)= \arg \max_{c \in \{-1,+1\}} \Pp(C=c | \bX = \bx ) =
\arg \max_{c \in \{-1,+1\}} \Pp(C=c , \bX = \bx )  .$$

\end{definition}

Examples of classical  generative classifiers are  the naive
Bayes~\citep{Minsky61stepstoward}, averaged one-dependence
estimators~\cite{Webb2005} and in general Bayesian network
classifiers~\citep{bielza2014}.  

The strictly positive assumption on generative classifiers permits defining a
real-valued discrimination function that sign-represents the binary classifier
$\phi_{\Pp}$ induced by the  generative classifier $\Pp$. 

\begin{definition} \label{def:discrimination}
The induced discrimination function $f_P$ of a generative classifier $\Pp$ is defined as

\begin{align*}
f_{\operatorname{P}}(\bx)=&
\ln\left(\frac{\operatorname{P}(\bX=\bx,C=+1)}{\operatorname{P}(\bX=\bx,C=-1)
}\right)   \\
=&\ln\left(\frac{\operatorname{P}(C=+1|\bX=\bx)}{\operatorname{P}(C=-1|\bX=\bx)}\right).
\end{align*}

\end{definition}

Let $\mathcal{P}$ be the set of all generative classifiers over predictors $\bX$
and $\mathcal{F}$ the set of functions $f:\bchi \to \mathbb{R}$.  For every
$f\in \mathcal{F}$, $\sign(f)$ is called the decision or  binary classifier
induced by $f$.  Let $\Phi$ be the mapping from $\mathcal{P}$ to $\mathcal{F}$
that assigns the induced discrimination function to a generative classifier
$\Pp$, that is, $\Phi(\Pp)=f_{\Pp}$.  For $f \in \mathcal{F}$, the level set
(fiber) $\Phi^{-1}(f)$ is the set of generative classifiers that induce $f$,
that is, the set of strictly positive probabilities $\Pp$ such that $f_{\Pp}=f$
(Figure~\ref{fig:map}). 

For every $A \subseteq [n]$ the set of functions
that depend only on the variables $\bx_A$ is indicated as,
$$\mathcal{F}_{A} = \left\{ f \in
\mathcal{F} \text{ s.t. } f(\bx_A,\bx_{-A}) = f(\bx_A, \bx'_{-A}) \text{ for all
} \bx \in \bchi,\, \bx'_{-A} \in \bchi_{-A} \right\}.$$
For $g \in
\mathcal{F}_{A}$ we can simply write $g(\bx_A)$ for $g(\bx_A,\bx_{-A})$. Obviously,
for every $A \subseteq [n]$, $\mathcal{F}_{A}$ is a linear subspace of
$\mathcal{F}$ of dimension
$|\bchi_A|$ and $\mathcal{F}_A \cap \mathcal{F}_B = \mathcal{F}_{A \cap B}$.
The space of constant functions, that
is, functions that do not depend on any variable, is indicated as
$\mathcal{F}_{\emptyset}$ and
$dim(\mathcal{F}_{\emptyset})=1$.

\begin{figure}
\centering
\begin{tikzpicture}
\node (lbl1) at (0,2.3) {$\mathcal{P}$};
\draw[black, fill=black] (0,-1.25) circle [radius=0.05]; 
\node (p1) at (0, -1) {$\Pp$};
\draw[] (0,0) circle [x radius=1.5, y radius=2];

\node (lbl2) at (4,2.3) {$\mathcal{F}$};
\node (f1) at (4, -1) {$f_{\Pp}$};
\draw[dotted] (-1,1.5) -- (1,1.5);
\draw[dotted] (-1.3,0.9) -- (1.3,0.9);
\node (fhim) at (-0.2,1.2) {$\Phi^{-1}(f)$};
\node (f) at (3.7, 1.3) {$f$};
\draw[black, fill=black] (4,-1.25) circle [radius=0.05]; 
\draw[black, fill=black] (3.7,0.9) circle [radius=0.05]; 
\draw[dotted,->] (fhim) to [out=40,in=140]  (f);
\draw (4,0) circle [x radius=1.5, y radius=2];
\draw[->] (lbl1) to [out=40,in=140] node[above] {$ \Phi$} (lbl2);
\draw[|->] (p1) to [out=20,in=160] (f1);

\end{tikzpicture}
\caption{The mapping $\Phi: \mathcal{P} \longrightarrow \mathcal{F}$. The image of a
generative classifier $\Pp$ is shown as $f_{\Pp}$, the level set of $f \in
\mathcal{F}$ is displayed with dotted lines.} \label{fig:map}
\end{figure}
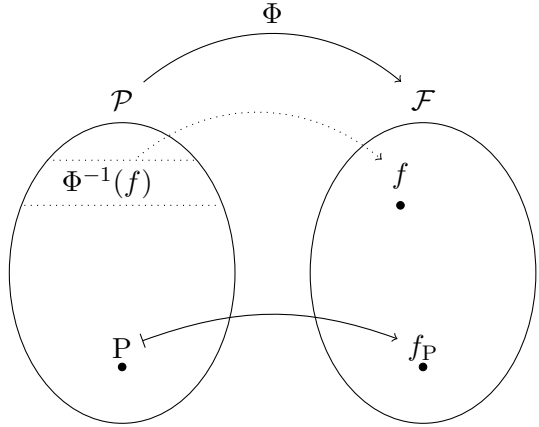

\section{Discrimination Functions Under Conditional Independences}
\label{sec:markov}

In this section we show that any conditional independence statement over the
random vector $(\mathbf{X},C) \in \bchi \times \{-1,1\}$ is equivalent to a set
of linear equations for the induced discrimination function. We then generalize
the result to sets of conditional independence statements induced by some
undirected graph, that is, undirected Markov models.

For $\Pp$ a joint probability distribution over $(\bX,C)$, and $A,B \subset [n]$
two disjoint subsets, let $\bX_A \indep \bX_B |\, (\bX_{-A\cup B},C)$ denote the
conditional independence of $\bX_A$ and $\bX_B$ given $\bX_{-A\cup B}$ and $C$.
Each conditional independence statement, for categorical random variables,
is equivalent to the corresponding toric equation for the probability density
function $p$~\citep{drton2009lectures},

\begin{equation} \label{eq:toric}
p(\bx_A,\bx_B,\bx_D,c)p(\bx'_A,\bx'_B,\bx_D,c) =  p(\bx_A,\bx'_B,\bx_D,c)p(\bx'_A,\bx_B,\bx_D,c),
\end{equation} 
for all $c\in \{-1,+1\}$, $\bx_B,\bx'_B \in \bchi_B$, $\bx_A,\bx'_A \in \bchi_A$
{and $\bx_D \in \bchi_D$ where  $ \bchi_D =  \bchi_{-A \cup B}$}.  Since we
assumed strict positivity of $p$ we can take the logarithm of
Equation~\eqref{eq:toric} and obtain a linear equation in the log probability
density function,
\begin{equation} \label{eq:lntoric}
\ln p(\bx_A,\bx_B,\bx_D,c)+ \ln
p(\bx'_A,\bx'_B,\bx_D,c) = \ln p(\bx_A,\bx'_B,\bx_D,c)+ \ln
p(\bx'_A,\bx_B,\bx_D,c). 
\end{equation}
In Section~\ref{sec:diffoper} we introduce a categorical difference operator centered in
$\bx^0 \in \bchi$ and acting on any function $f \in \mathcal{F}$. In
Section~\ref{sec:condindep} we will show how Equation~\eqref{eq:lntoric} can be
written using this difference operator. This leads to a characterization
of discrimination functions under conditional independence assumptions and in
general under the  undirected Markov property which can be synthetically
expressed using the categorical difference operator. 

\subsection{Difference Operator} \label{sec:diffoper}

\begin{definition} \label{def:difference}
For $f \in \mathcal{F}$ and $A\subseteq[n]$, the {first order}
$A$-difference of {$f$} 
 (centered at $\bx_A^0 \in \bX_A$) is defined as 
$$ \Delta^{\bx_A^0}_{A}f(\bx) = f(\bx) - f(\bx^0_A, \bx_{-A}) .$$ 
\end{definition}
When $A=\{i\}\subset [n]$ we will simply write $\Delta^{x_i^0}_i$ to denote $\Delta^{x_i^0}_{A}$.

Difference operators of order greater than one can be defined iteratively.
In particular, for $A,B\subset [n]$ we are interested in the second order difference
\begin{align}
\Delta^{\bx_A^0}_{A}\Delta^{\bx_B^0}_{B} f(\bx) = & \Delta^{\bx_A^0}_A(f(\bx)-f(\bx_{B}^{0},\bx_{-B})) \nonumber \\ 
                                       = & f(\bx) + f(\bx_{A\cup B}^{0},\bx_{-(A\cup B)}) - f(\bx_{A}^{0},\bx_{-A})- f(\bx_{B}^{0},\bx_{-B})  .\label{eq:secondorder}
\end{align} 

\begin{example} \label{exp:1}
\textrm 
Let $f$ be the function over the two predictors $(X_1,X_2) \in \{0,1\} \times \{1,2,3\}$ shown in Table~\ref{tab:ex1f}.
The second order difference $\Delta_1^0\Delta_2^1 f $ is reported in
Table~\ref{tab:ex1f}. For example, $\Delta_1^0\Delta_2^1 f(1,2)= f(1,2) + f(0,1)
- f(0,2) - f(1,1) = -16 $. 
\begin{table}
\centering
\caption{Discrimination function $f$ and its second order difference of Example~\ref{exp:1}.}
\label{tab:ex1f}
\hspace{2pt}

\begin{tabular}{cc|ccc}
$f$  & & \multicolumn{3}{c}{$X_2$} \\  
    &   &   1 & 2 & 3 \\
 \hline 
\multirow{2}{*}{$X_1$}& 0 & -1& 5& 2 \\
& 1 & 3& -7& -4
\end{tabular} 
\hspace{15pt}
\begin{tabular}{cc|ccc}
$\Delta^{0}_1\Delta^{1}_2f$  & & \multicolumn{3}{c}{$X_2$} \\  
    &   &   1 & 2 & 3 \\
 \hline 
\multirow{2}{*}{$X_1$}& 0 & 0& 0& 0 \\
                      & 1 & 0& -16& -10 
\end{tabular}
\end{table}
\end{example}

Lemma~\ref{lem:x0x1} connects the difference operators centered at different points.
\begin{lemma}
\label{lem:x0x1}
Let $f \in \mathcal{F}$, $\bx_A^0, \bx_A^1 \in \bchi_A,\bx_B^0, \bx_B^1 \in
\bchi_B,  $  and $A,B \subseteq [n]$. Then
\begin{itemize}
\item[(i)] $\Delta^{\bx_A^1}_{A}f(\bx)-  \Delta^{\bx_A^0}_{A}f(\bx) =
\Delta^{\bx_A^1}_{A}f(\bx^0_A, \bx_{-A})$,
\item[(ii)]  $\Delta^{\bx_A^0}_{A}f(\bx) = 0 $ for all $\bx \in \bchi $ if and
only if $ \Delta^{\bx_A^1}_{A}f(\bx) = 0 $ for all $\bx \in \bchi $,
\item[(iii)] $ \Delta^{\bx_A^0}_{A}\Delta^{\bx_B^0}_{B}f(\bx) = 0 $ for all
$\bx \in \bchi $ if and only if $ \Delta^{\bx_A^1}_{A}\Delta^{\bx_B^1}_{B}f(\bx) = 0 $
for all $\bx \in \bchi $.
\end{itemize} 
\end{lemma}
\begin{proof} Items \textit{(ii)} and \textit{(iii)} follow from \textit{(i)}. For proving \textit{(i)} we use Definition~\ref{def:difference}
\begin{align*}
\Delta^{\bx_A^1}_{A}f(\bx)-\Delta^{\bx_A^0}_{A}f(\bx) = &f(\bx) - f(\bx_{A}^{1},\bx_{-A}) - f(\bx) + f(\bx_{A}^{0},\bx_{-A})  \\
 =&f(\bx_{A}^{0},\bx_{-A})- f(\bx_{A}^{1},\bx_{-A})    
 =\Delta_{A}^{\bx^1}f(\bx_A^{0},\bx_{-A}).
\end{align*}
\end{proof}

Because of Lemma~\ref{lem:x0x1} we can assume $\bx_A^0$ fixed and write $\Delta_{A}$ for $\Delta^{\bx_A^0}_{A}$.
Furthermore if $f(\bx)=0$ for all $\bx \in \bchi$, we write $f\equiv 0$.
The following lemma, whose proof follows directly from Definition~\ref{def:difference}, collects the basic properties of $\Delta_{A}$.
\begin{lemma} \label{lem:first}
Let $f,g \in \mathcal{F}$  and $A
\subseteq[n]$. Then
\begin{itemize}
\item[(i)] $\Delta_{A}f(\bx) \equiv 0$ if and only if  $f \in \mathcal{F}_{A}$.
\item[(ii)]  $f(\bx)=h(\bx_{-A}) + \Delta_Af(\bx)$ for all $\bx \in \bchi$, where $h \in \mathcal{F}_{-A}$.
\item[(iii)] $\Delta_{A}(\alpha f+\beta g)(\bx) = \alpha \Delta_{A}f(\bx) + \beta\Delta_{A}g(\bx)$ for all $\bx \in \bchi$ and for all $\alpha,\beta \in \mathbb{R}$. 
\end{itemize} 
\end{lemma}

{Lemma~\ref{lem:second} collects useful properties of the second order differences.}
\begin{lemma} \label{lem:second}
Let $f \in \mathcal{F}$  and $A,B \subseteq [n]$. Then 
\begin{itemize}
\item[(i)] $ \Delta_{A}\Delta_{B} f(\bx) = \Delta_{A}f(\bx)  + \Delta_{B}f(\bx) -\Delta_{A \cup B}f(\bx)  $ for all $\bx \in \bchi$.
\item[(ii)] $\Delta_{A}\Delta_{A}f(\bx) =  \Delta_{A}f(\bx)$ for all $\bx \in \bchi$.
\item[(iii)]  $\Delta_{A} \Delta_{B}f(\bx) \equiv 0$ if and only if there exist
a function $h \in \mathcal{F}_{-A}$  and a function  $g \in \mathcal{F}_{-B}$ such that $f(\bx) = h(\bx_{-A}) + g(\bx_{-B})$ for all $\bx \in \bchi$.
\end{itemize}
\end{lemma}
\begin{proof}
Points \textit{(i)} and \textit{(ii)} follow directly by
Definition~\ref{def:difference}.  To prove point \textit{(iii)} we just observe
that from point \textit{(i)}  of Lemma~\ref{lem:first}  we have that
$$ \Delta_{A} \Delta_{B}f(\bx) = 0 \text{ for all } \bx \in \bchi
\text{ if and only if } \Delta_{B}f(\bx) = h(\bx_{-A}),$$
and thus by point \textit{(ii)} of Lemma~\ref{lem:first}, $f(\bx) = g(\bx_{-B}) + h(\bx_{-A})$
for all $\bx \in \bchi$.
\end{proof}

For $A \subseteq [n]$ the $A$-difference operator of $f$ cannot be a non zero function of variables $\bx_{-A}$ solely, as shown in Lemma~\ref{lem:diffvar}.
\begin{lemma} \label{lem:diffvar}
Let $f\in \mathcal{F}$ and $A \subseteq [n]$. If $\Delta_Af(\bx) \in \mathcal{F}_{-A}$, then $\Delta_Af \equiv 0$.
\end{lemma}
\begin{proof}
From point \textit{(ii)} of Lemma~\ref{lem:first} we have that 
$$ f(\bx) = \Delta_Af(\bx) + h(\bx_{-A}).$$
Thus now applying point \textit{(i)} of Lemma~\ref{lem:first} we obtain 
$$ \Delta_A f(\bx) = \Delta_A (\Delta_Af(\bx) + h(\bx_{-A}))=0.$$  
\end{proof}

\subsection{Conditional Independence} \label{sec:condindep}
The logarithm of the toric equation
(Equation~\eqref{eq:lntoric}) can be written as the second order
difference operator (Equation~\eqref{eq:secondorder}) of the logarithm
probability density function equal to zero. 
\begin{remark} \label{obs:toric}
Consider $\Pp \in \mathcal{P}$ with density $p$ such that $\bX_A \indep \bX_B |\,  ( \bX_{-A \cup B}, C )  $ then Equation~\eqref{eq:toric} is equivalently written as
$$ \Delta_A \Delta_B \ln   p(\bx , c )   \equiv 0,  \quad \text{for all } \; c \in \{-1,+1\}.$$
\end{remark}

A conditional independence statement among the predictor variables is equivalent
to the related second order difference of the discrimination function being
equal to zero as shown by Lemma~\ref{lem:cisdisc}.
\begin{lemma} \label{lem:cisdisc}
Let $\mathbf{X}=(X_1,\ldots,X_n)$ be a predictor vector of discrete random variables and $C$ a binary class variable.  Let $A,B,D$ {be} a partition of $[n]$ and $f \in \mathcal{F}$.
The following statements are equivalent 
\begin{itemize}
\item[(i)] there exists a generative classifier $\Pp \in \Phi^{-1}(f)$ such that $\bX_A \indep \bX_B |\, (\bX_{D}, C)$ holds 
\item[(ii)] $\Delta_{A}\Delta_{B}f\equiv 0$.
\end{itemize}
\end{lemma}
\begin{proof}
First we prove that \textit{(i)} implies \textit{(ii)}. 
Let $\Pp$ be a probability distribution  
such that $\bX_A \indep \bX_B |\, (\bX_D, C)$ and $f(\bx)=f_{\Pp}(\bx)$. Then by Remark~\ref{obs:toric}
$$ \Delta_A \Delta_B \left( \ln p(\bx,  c) \right) = 0 \quad \text{for all } \bx \in \bchi\text{, }c \in \{-1,+1\} .$$ 
From the linearity of $\Delta_A \Delta_B$ and {the fact that}  $f_P(\bx) =\ln
p(\bx,+1)-\ln p(\bx,-1)$  we obtain \textit{(ii)}.

Conversely, we need to define $\Pp$, with density $p$, such that $\Delta_A \Delta_B \ln p(\bx , c) = 0$ for all $\bx \in \bchi$ and $c \in \{-1,+1\}$, and that satisfies $f_{\Pp}(\bx) = f(\bx)$ for all  $\bx \in \bchi$.
For any  $g(\bx): \bchi 
{ \longrightarrow} \mathbb{R}$ such that $\Delta_A\Delta_B g \equiv 0$ (for example $g \equiv 0$), 
a suitable $p$ is defined as
$$ \ln p(\bx, c) = g(\bx) + k + \frac{c}{2}f(\bx) $$
where $k$ is an appropriate normalization constant, that is,
$$ \sum_{\bx \in \bchi}  \exp(g(\bx)) \left( 1+ e^{f(\bx)} \right) = \exp(-k)  \textcolor{blue}{.}   $$
\end{proof}

\subsection{Markov classifiers} \label{sec:markovclassifier}
In this section we consider generative classifiers such that the underlying   
probability satisfies the undirected Markov property with respect to a given
graph. To define Markov classifiers we recall some basic notions on
separation and connectivity on graphs.  Given an undirected graph $\mathcal{G}$
with node set $V$ and edges $E$,  two nodes $a, b \in V$ are adjacent if
the edge $\{a,b\} \in E$. Given  $A\subset V$, the subgraph induced by $A$,
denoted $\mathcal{G}_A$, is said to be a complete subgraph if every pair of
nodes $a,b \in A$ are adjacent. Complete subgraphs that are maximal
with respect to inclusion are called \textit{cliques}. The set of  cliques of
the graph $\mathcal{G}$ is denoted $\mathcal{C}(\mathcal{G})$.  Moreover given
three disjoint subsets of nodes $A,B,D \subset V$, $D$ separates $A$ and $B$
in $\mathcal{G}$ if every path that connects $A$ to $B$ in $\mathcal{G}$ passes
through $D$.

\begin{definition}
Let $\mathcal{G}$ be 
an undirected graph over nodes $[n]$.
A $\mathcal{G}$-Markov classifier is a generative classifier $\Pp \in \mathcal{P}$ such that
$$  X_i \indep X_j |\, (\bX_{-\{i,j\}}, C) \quad \text{for any pair of
    non-adjacent nodes } i,j. $$ 
\end{definition}
Alternatively, we can define a $\mathcal{G}$-Markov classifier as a generative
classifier $\Pp \in \mathcal{P}$ that satisfies pairwise (or equivalently global
or local since $\Pp>0$) Markov property~\citep{lauritzen1996} with respect to an
extended undirected graph; the extended graph is defined adding a node
corresponding to $C$ in the graph $\mathcal{G}$ and connecting $C$ to all the
nodes of the predictor variables (see Figure~\ref{fig:markovclass}). Thus, if
$\Pp$ is a $\mathcal{G}$-Markov classifier the following conditional
independence statement holds for every $A,B \subset [n]$ such that $- (A \cup
B)$ separates $A$ and $B$ in $\mathcal{G}$,
\begin{equation} \label{eq:cismarkovclassifier}
\bX_A \indep \bX_B | \, (\bX_{-A\cup B},C).
\end{equation}
The set of $\mathcal{G}$-Markov classifiers will be denoted by $\mathcal{P}_{\mathcal{G}}$.

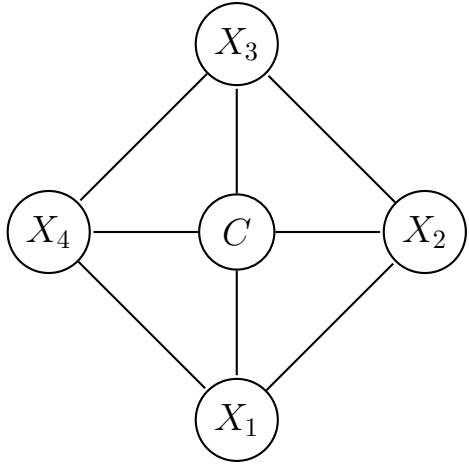
\begin{figure}
\centering
\begin{tikzpicture}
[-,>=stealth',shorten >=1pt,auto,node distance=1cm,
  thick,main node/.style={circle,fill=white!10,draw,font=\sffamily\Large\bfseries},minimum width=1cm]  
  \node[main node] (0) at (0,0) {$C$};
  \node[main node] (1) at (0,-2.5) {$X_1$};
  \node[main node] (3) at (0,2.5) {$X_3$};
  \node[main node] (2) at (2.5,0) {$X_2$};  
  \node[main node] (4) at (-2.5,0) {$X_4$};  
  \path[every node/.style={font=\sffamily\small}]
    (0) edge node [left] {}  (2)
        edge node [left] {}  (1)
         edge node [left] {}  (4)
        edge node[left] {} (3)
     (1) edge (2)
     (2) edge (3)
     (3) edge (4)  
     (4) edge (1);
\end{tikzpicture}
\caption{Example of the structure of a Markov classifier, where the node
    corresponding to  $C$ is adjacent to all the predictor nodes.}
\label{fig:markovclass}
\end{figure}

For $\mathcal{G}$-Markov classifiers the following result similar
to the Hammersley-Clifford theorem~\citep{hammersley1971, grimmet1973,
gandolfi2017} holds. The proof of Theorem~\ref{theo:discrete} is partly based on the
M\"{o}ebius inversion lemma~\citep{Rota1987} as in the proof of the
Hammersley-Clifford theorem~\citep{grimmet1973, lauritzen1996}.

\begin{theorem}
\label{theo:discrete}
For any discrimination function $f  \in \mathcal{F}$ and any undirected graph
    $\mathcal{G}$ over $X_1,\ldots,X_n$ the following statements are equivalent:
\begin{itemize}
    \item[(i)] There exists a $\mathcal{G}$-Markov classifier $\Pp \in
        \Phi^{-1}(f)$, 
    \item[(ii)] $\Delta_A\Delta_Bf\equiv0$ for every $A,B$ separated
        by $-(A\cup B)$ in $\mathcal{G}$, 
    \item[(iii)] $f(\bx) = \sum_{A\subseteq [n] } g_A(\mathbf{x}_A)$, such that
        $g_A \in \mathcal{F}_{A}$ and  $g_A \equiv 0$ if $\mathcal{G}_A$ is not complete.
\end{itemize}
\end{theorem}
\begin{proof}
Item \textit{(i)} implies directly \textit{(ii)} by
Lemma~\ref{lem:cisdisc} applied to every conditional independence statements in Equation (\ref{eq:cismarkovclassifier}).

To prove that \textit{(ii)} implies \textit{(iii)} let $V_A(\mathbf{x}_A)= f(\bx_A, \bx^0_{-A}) $ and
$$ g_A(\mathbf{x}_A) = \sum_{B\subseteq A} (-1)^{|A\setminus B|}V_{B}(\mathbf{x}_B).$$

From M\"{o}ebius inversion lemma~\citep{Rota1987}  
we have that
    $$f(\mathbf{x})=V_{[n]}(\bx)=\sum_{A\subseteq [n]} g_A(\mathbf{x}_A).$$
We just have to show that $g_A \equiv 0$ if $\mathcal{G}_A$ is not complete.
Let $A$ be a subset of $[n]$  such that $\mathcal{G}_A$ is not complete, then
    there exist $a,b \in A$ such that $a$ and $b$ are not adjacent. Thus we can
    write, for $D=A\setminus\{a,b\}$,
{\small
\begin{align*}
g_A(\mathbf{x}_A) = & \sum_{B\subseteq D} (-1)^{|D\setminus
    B|}\left(V_{B}(\mathbf{x}_B)-V_{B\cup \{a\}}(\mathbf{x}_B,x_a) - V_{B\cup
    \{b\}}(\mathbf{x}_B,x_b) + V_{B\cup \{a,b\}}(\mathbf{x}_B,x_a,x_b) \right)  =\\ = &
\sum_{B\subseteq D} (-1)^{|D\setminus B|} \Delta_{a} \Delta_{b} f =0
\end{align*}
}
where the last equality is due to the fact that $\Delta_a\Delta_b f \equiv 0$ by 
item \textit{(ii)} since $a$ and $b$ are not adjacent.

Finally to prove that \textit{(iii)} implies \textit{(i)} let $\operatorname{P}$ be
a generative classifier with density:
\begin{equation} \label{eq:param1}
p(x_1,\ldots,x_n, c) =  \exp\left(K + \frac{c}{2}f(\bx) + g(\bx) \right)
\end{equation}
where $g(\bx)= \sum_{A \subseteq [n]} g_A(\bx_A)$ and  $g_A \equiv 0$ when $\mathcal{G}_A$ is not complete and $K$ is the appropriate normalization constant.
Item \textit{(iii)} implies item \textit{(ii)} by Lemma~\ref{lem:second} and
    thus, obviously, $\Delta_i \Delta_j \ln p \equiv 0$ for every $i,j \in [n]$
    not adjacent in $\mathcal{G}$. Thereafter, $\Pp$ is a $\mathcal{G}$-Markov classifier by Remark~\ref{obs:toric}.
The induced discrimination function $f_{\Pp}$ is clearly equal to $f$, for every choice of $g$, proving \textit{(i)}.

\end{proof}

When a function $f \in \mathcal{F}$ satisfies point \textit{(ii)} in
Theorem~\ref{theo:discrete} for a given graph $\mathcal{G}$ we will concisely write $\Delta^{2}_{\mathcal{G}}f\equiv 0$. 
Moreover we can observe that the factorization in point \textit{(iii)} can be  written using only 
the set of cliques, $\mathcal{C}(\mathcal{G})$, of the graph $\mathcal{G}$.
Thus Theorem~\ref{theo:discrete} can be summarized by the following equivalences,
$$  \Pp \in \mathcal{P}_{\mathcal{G}} \cap \Phi^{-1}(f)  \text{ if and only if } \Delta_{\mathcal{G}}^{2}f \equiv 0 \text{ if and only if } f(\bx)=\sum_{A \in \mathcal{C}(\mathcal{G})} g_A(\bx_A) .$$

\begin{definition}
The set of $f \in \mathcal{F}$ such that $\Delta^{2}_{\mathcal{G}}f\equiv 0$ is denoted as $\mathcal{F}_{\mathcal{G}}$.
\end{definition}

By Theorem~\ref{theo:discrete} a function $f\in \mathcal{F}_{\mathcal{G}}$ is the sum of functions $g_A \in \mathcal{F}_A$  depending only on the predictors $\bX_A$ in a clique  $A \in \mathcal{C}(\mathcal{G})$.

\subsection{Complexity of Markov Classifiers} \label{sec:complexity}

Following  \cite{ling2003} we use the contained XORs to measure the complexity of binary classifiers and thus to quantify their expressive power. 

\begin{definition} \label{def:xorcont}
Let $\phi: \bchi \rightarrow \{-1,+1\}$ be a binary classifier over predictors
    $\bX$ and $A\subseteq [n]$. The function  $\phi$ is said to contain an $A$-XOR (or an XOR
    among variables $\bx_A$) if there exist $\bx_{-A}^0 \in \bchi_{-A}$ and
    $\dot{\bx}_A, \ddot{\bx}_A \in \bchi_A$, such that,
    $$ \phi(\bx_A,\bx^0_{-A}) = \prod_{i\in A} (-1)^{\delta_{\dot{x}_i}(x_i)}
    \text{ for all } \bx_A \in \bigtimes_{i \in A} \{ \dot{x}_i, \ddot{x}_i \},$$ 
    where $\delta_{\dot{x_i}}(x_i)=1$ if $x_i = \dot{x_i}$ and $0$ otherwise. 
\end{definition} 
Definition~\ref{def:xorcont} extends the concept of \textit{Boolean XOR} (or
\textit{logical parity} function)~\citep{odonnell2014} to binary functions over
categorical variables. Indeed, if $\phi$ contains an $A$-XOR as in
Definition~\ref{def:xorcont} then the following is a Boolean XOR function over
$|A|$ Boolean variables,
$$
\begin{array}{ccccc}
\psi \, : &  \{-1,+1\}^{|A|} &   \longrightarrow &         \{-1,+1\} \\ 
          &    \mathbf{k}     &   \longmapsto & \phi(\tau(\mathbf{k}), \bx_{-A}^0) &= \prod_{i} k_i
\end{array}
$$ 
where $\tau(\mathbf{k}):\{-1,+1\}^{|A|} \rightarrow  \bigtimes_{i \in A} \{ \dot{\bx}_i, \ddot{\bx}_i \}$ is the coordinate-wise mapping that assigns $-1$ to $\dot{x}_i$ and $+1$ to $\ddot{x}_i$.  

Definition~\ref{def:xorcont} implies  that if $\phi$ contains an $A$-XOR  then $\phi$ contains a $B$-XOR  for every $B \subseteq A$. Moreover, if $\phi$ contains an $A$-XOR we have that, for some $\dot{\bx}_A, \ddot{\bx}_A \in \bchi_A$ and $j \in A$, 
\begin{equation} \label{eq:xoreq}
\phi(\dot{x}_j, \bx_{A\setminus \{j\}}, \bx^{0}_{-A}) = - \phi(\ddot{x}_j, \bx_{A\setminus \{j\}}, \bx^{0}_{-A}) \text{ for all } \bx_{A\setminus \{j\} } \in \bigtimes_{i \in A\setminus \{j\}} \{ \dot{x}_i, \ddot{x}_i \}.
\end{equation}

As a corollary of Theorem~\ref{theo:discrete}, the decision function
of a $\mathcal{G}$-Markov classifier can contain XORs just among variables that
belong to a clique of $\mathcal{G}$.  
\begin{corollary} \label{cor:xorG1}
If $f \in \mathcal{F}_{\mathcal{G}}$  and $\phi=\sign(f)$ contains an $A$-XOR, then $\mathcal{G}_A$ is a complete subgraph of $\mathcal{G}$.
Equivalently, if $\mathcal{G}_A$ is not complete, 
then there does not exist a $\mathcal{G}$-Markov classifier such that its induced decision function contains an $A$-XOR. 
\end{corollary} 
\begin{proof}
Consider a graph $\mathcal{G}$ and $A \subseteq [n]$ such that the induced sub-graph is not complete. 
Thus there exist  $i,\, j \in A$  non adjacent in $\mathcal{G}$. 
Thereafter, from Theorem~\ref{theo:discrete}, we have that for every $f \in \mathcal{F}_{\mathcal{G}}$,
$\Delta_i \Delta_j f \equiv 0$. Expanding the definition of second order difference (Equation~\eqref{eq:secondorder}) and using point \textit{(iii)} of Lemma~\ref{lem:x0x1}, we obtain
\begin{equation} \label{eq:expanded}
 f(\bx) + f(x'_i, x'_j,\bx_{-\{i,j\}}) = f(x'_i,\bx_{-i}) + f( x'_j, \bx_{-j}) \quad \text{for all } \bx \in \bchi,\, x'_i \in \nchi_i,\, x'_j \in \nchi_j.
 \end{equation}

Suppose now that $\phi= sign(f)$ contains an $\{i,j\}$-XOR. From Equation~\eqref{eq:xoreq} we have that, for some $\dot{x}_i, \ddot{x}_i \in \nchi_i, \, \dot{x}_j, \ddot{x}_j \in \nchi_j$  and  $\bx^{0}_{-\{i,j\}} \in \bchi_{-\{i,j\}}$,
$$
 -\phi(\dot{x}_i, \ddot{x}_j, \bx^0_{-\{i,j\}}) =\phi(\dot{x}_i, \dot{x}_j, \bx^0_{-\{i,j\}}) = - \phi(\ddot{x}_i, \dot{x}_j, \bx^0_{-\{i,j\}}) = \phi(\ddot{x}_i, \ddot{x}_j, \bx^0_{-\{i,j\}}).
 $$
This is absurd since Equation~\eqref{eq:expanded} holds.

We have proven that $\phi$ does not contain an XOR among $(x_i,x_j)$, thus much less it can contain an XOR among $\bx_A$. 
\end{proof}

Similarly to  \cite{varando15}, we can bound the number of decision functions
representable by a $\mathcal{G}$-Markov classifier for a fixed graph
$\mathcal{G}$. Let $sign(\mathcal{F}_{\mathcal{G}}) = \{ sign(f) \text{ s.t. } f
\in \mathcal{F}_{\mathcal{G}} \}$ be the set of decision functions representable
by a $\mathcal{G}$-Markov classifier, the following Corollary~\ref{cor:bound}  follows immediately from Theorem~\ref{theo:discrete} and the same argument as in~\cite{varando15}.

\begin{corollary} \label{cor:bound}
Let $\mathcal{G}$ be an undirected graph with cliques $\mathcal{C}(\mathcal{G})$ then
$$ |sign(\mathcal{F}_{\mathcal{G}})| \leq 2\sum_{k=0}^{d-1} \binom{|\bchi|-1}{k} ,$$
where $d= dim(\mathcal{F}_{\mathcal{G}}) $.
\end{corollary} 

Moreover, for decomposable graphs~\citep{lauritzen1996}, the dimension of $\mathcal{F}_{\mathcal{G}}$ can be iteratively computed  as $dim(\mathcal{F}_{\mathcal{G}}) = dim(\mathcal{F}_{\mathcal{G}_{A\cup D}}) + dim(\mathcal{F}_{\mathcal{G}_{B \cup D}}) - |\bchi_{D}|$ where $(A,B,D)$ is a decomposition of $\mathcal{G}$. This fact follows directly observing that, if $(A,B,D)$ is a decomposition of $\mathcal{G}$, $\mathcal{C}(\mathcal{G}) = \mathcal{C}(\mathcal{G}_{A\cup D}) \cup \mathcal{C}(\mathcal{G}_{B\cup D})$ and thus, $\mathcal{F}_{\mathcal{G}} = \mathcal{F}_{\mathcal{G}_{A\cup D}} + \mathcal{F}_{\mathcal{G}_{B\cup D}} $ .

When the graph is not decomposable it is still possible to compute the dimension of $\mathcal{F}_{\mathcal{G}}$ using 
that it is a sum of the spaces $\mathcal{F}_A$ for $A \in \mathcal{C}(\mathcal{G})$ (see Example~\ref{exp:2}).

Alternatively, \cite{LI2018} developed an algebraic way to compute the dimension of
$\mathcal{F}_{\mathcal{G}}$ for Markov networks.
As an application, they bounded the generalization error using the
Vapnik-Chervonenkis theory~\citep{vapnik}. 
\begin{example}\label{exp:2}
\textrm 
Consider the $\mathcal{G}$-Markov classifier with structure as in
    Figure~\ref{fig:markovclass}, and let assume that $|\nchi_i|=2$ for $i \in
    [4]$, that is, all the predictor variables are binary.

The cliques of the graph $\mathcal{G}$ are  $\mathcal{C}(\mathcal{G})=\{ \{1,2\}, \{2,3\}, \{3,4\}, \{4,1\} \}$. Thus, by Theorem~\ref{theo:discrete}, $\mathcal{F}_{\mathcal{G}} = \mathcal{F}_{\{1,2\}}+ \mathcal{F}_{\{2,3\}} + \mathcal{F}_{\{3,4\}} + \mathcal{F}_{\{4,1\}}$.
To compute the dimension of $\mathcal{F}_{\mathcal{G}}$, observe that
\begin{align}
    dim(\mathcal{F}_{\mathcal{G}}) =& dim\left((\mathcal{F}_{\{1,2\}}+ \mathcal{F}_{\{2,3\}} + \mathcal{F}_{\{3,4\}}) + \mathcal{F}_{\{4,1\}}\right) \label{eq:exp1}\\
    =& dim(\mathcal{F}_{\{1,2\}}+ \mathcal{F}_{\{2,3\}} + \mathcal{F}_{\{3,4\}}) + dim(\mathcal{F}_{\{4,1\}}) - \nonumber\\ 
  & dim( (\mathcal{F}_{\{1,2\}}+ \mathcal{F}_{\{2,3\}} + \mathcal{F}_{\{3,4\}}) \cap \mathcal{F}_{\{4,1\}}) \nonumber
\end{align} 
Observe that since $\mathcal{F}_{\{4,1\}}$ are functions that depends only on $X_1$ and $X_4$,
$$ (\mathcal{F}_{\{1,2\}}+ \mathcal{F}_{\{2,3\}} + \mathcal{F}_{\{3,4\}}) \cap \mathcal{F}_{\{4,1\}}= \mathcal{F}_{\{1\}} + \mathcal{F}_{\{4\}} .$$
Moreover, since $\mathcal{F}_{\{1\}} \cap \mathcal{F}_{\{4\}} = \mathcal{F}_{\emptyset}$, 
$$dim(\mathcal{F}_{\{1\}} + \mathcal{F}_{\{4\}})=dim(\mathcal{F}_{\{1\}}) + dim(\mathcal{F}_{\{4\}}) - dim(\mathcal{F}_{\{1\}} \cap \mathcal{F}_{\{4\}})= 3.$$
Using now the fact that $dim(\mathcal{F}_{A})= |\bchi_A| = 4 $ for every $A \in \mathcal{C}(\mathcal{G})$ we obtain, from Equation~\eqref{eq:exp1},
$$dim(\mathcal{F}_{\mathcal{G}})= dim\left(\mathcal{F}_{\{1,2\}}+ \mathcal{F}_{\{2,3\}} + \mathcal{F}_{\{3,4\}}\right) + 1.$$
With similar arguments we obtain,
$$ dim\left(\mathcal{F}_{\{1,2\}}+ \mathcal{F}_{\{2,3\}} + \mathcal{F}_{\{3,4\}}\right)= dim\left(\mathcal{F}_{\{1,2\}}+ \mathcal{F}_{\{2,3\}}\right) + 2 ,$$
and
$$dim\left(\mathcal{F}_{\{1,2\}}+ \mathcal{F}_{\{2,3\}}\right)= dim(\mathcal{F}_{\{1,2\}})+  dim(\mathcal{F}_{\{2,3\}})-dim(\mathcal{F}_{2})= 6$$
Thereafter, finally,
$$ dim(\mathcal{F}_{\mathcal{G}})= 6+2+1=9.$$ 
Thus from Corollary~\ref{cor:bound} we obtain the following upper bound on the
    number of decision functions representable by $\mathcal{G}$-Markov
    classifiers with structure as in Figure~\ref{fig:markovclass} and binary
    predictors:
$$ |sign(\mathcal{F}_{\mathcal{G}})| \leq 45638,$$
    this corresponds to approximately $70\%$ of the total $2^{16}$ decision functions over $\bchi$.
\end{example}

\subsection{Application to Bayesian Network Classifiers} 

Bayesian network classifiers~\citep{bielza2014} are one of the most common examples of generative classifier. 
In a Bayesian network classifier the joint probability distribution  of the predictor variables $\bX$ and the class $C$ is represented by a Bayesian network.
A Bayesian network \citep{Pearl1988} is composed by a directed acyclic graph
(DAG) with nodes indexed as the random variables and by a probability distribution factorized as
$$ \Pp(\mathbf{X})= \prod_{i=1}^{n} \Pp\left(X_i|\bX_{\operatorname{pa}(i)}\right),$$
where $\operatorname{pa}(i) \subset [n]$ denotes the parents of $i$ in the DAG
$\mathcal{G}$.
In particular Bayesian augmented naive Bayes (BAN) classifiers are Bayesian network classifier where the structure is composed of a DAG among the predictor variables and  
the class variable $C$ is added as a father of all the predictors.

A Bayesian network without $V$-structures, that is, where every set of parents is completely connected, is equivalent to an undirected Markov model with the same structure (dropping the directions of the arcs)~\citep{lauritzen1996}. Thus, Theorem~\ref{theo:discrete} applies directly to every BAN classifier without $V$-structures.
Moreover if the BAN classifier has some $V$-structures, we can still apply Theorem~\ref{theo:discrete} to the corresponding moral graph.
In particular, if $\Pp$ is the probability distribution of a BAN classifier, we
obtain that the induced discrimination function factorizes as follows,
$$ f_{\Pp}(\bx) = \sum_{i\in [n]} g_{i}(\bx_{\operatorname{pa}(i)},x_i), $$
where every function $g_i$ depends on the variable $X_i$ and the
$\bX_{\operatorname{pa}(i)}$.

Those results were obtained directly by \cite{varando15} for BAN classifiers.
Similarly, Corollary~\ref{cor:xorG1} can be applied to Bayesian network classifiers, considering the corresponding undirected moral graph obtaining in this case the results of~\cite{ling2003}.

\section{Parameter Estimation}
\label{sec:learning}

The method generally used to fit the parameters of a generative classifier is
maximum-likelihood estimation. But if the model is
misspecified, that is, the true distribution does not satisfy the same conditional independences of the model, the maximum-likelihood estimation is not optimal \citep{ptpr}. This is true, as shown by \cite{Domingos97} for naive Bayes, also if the true discrimination function belongs to $\mathcal{F}_{\mathcal{G}}$. 

Discriminative learning algorithms, that is, procedures that search directly a space $\mathcal{F}' \subset \mathcal{F}$  for the \textit{best} (for example with respect to empirical error minimization or conditional likelihood) discrimination function given a dataset, do not suffer from the same problem of generative learning algorithms~\citep{Ng2001,ptpr}.     

In this section we show how  to find the maximum-likelihood $\mathcal{G}$-Markov classifier that induces a given decision function $f \in \mathcal{F}_{\mathcal{G}}$.
We thus suggest a way of joining the generative and discriminative approaches.

\subsection{Fixed Discrimination Maximum-Likelihood Estimator} \label{sec:mle}
Fix an undirected graph $\mathcal{G}$ and a discrimination function $f \in \mathcal{F}_{\mathcal{G}}$. 
From Theorem~\ref{theo:discrete}, there exists a $\mathcal{G}$-Markov
classifier, $\Pp$, that induces $f$ (that is $\Phi(\Pp)=f_{\Pp}=f$). 
Actually from the proof of Theorem~\ref{theo:discrete} it follows that there
exists a whole family of $\mathcal{G}$-Markov classifiers that induce $f$.  

We are interested now in obtaining the generative classifier that maximizes the likelihood over a dataset $\mathcal{D}$ among such family. 
From the proof of Theorem~\ref{theo:discrete} we have that the density $p$ of
$\Pp \in \mathcal{P}_{\mathcal{G}}(f) =  \mathcal{P}_{\mathcal{G}} \cap
\Phi^{-1}(f)$ can be parametrized as described in Equation\eqref{eq:param1}: 
$$  p(\bx,c) = \exp\left(K + g(\bx) + \frac{c}{2}f(\bx) \right) \text{ for  } g \in \mathcal{F}_{\mathcal{G}}. $$

Similarly to the maximum-likelihood estimation in Markov models, we need to complete the set $\mathcal{P}_{\mathcal{G}}(f)$ with the limiting distributions.

\begin{definition}
We say that $\Pp$ is a marginally extended $\mathcal{G}$-Markov classifier if
    there exists $\Pp_{n} \in \mathcal{P}_{\mathcal{G}}(f)=\mathcal{P}_{\mathcal{G}} \cap \Phi^{-1}(f)$ such that
 $$  \Pp(\bX=\bx, C = c) = \lim _{n \to \infty} \Pp_{n}(\bX=\bx, C = c) \quad \text{ for all } \bx \in \bchi \text{ and } c \in \{-1,+1\}.$$ 
\end{definition} 

With $\overline{\mathcal{P}}_{\mathcal{G}}(f)$ we denote the family of
marginally extended Markov classifiers with discrimination function $f$ and consider the maximum-likelihood problem over the set of marginally
extended Markov classifiers: 
\begin{equation} \label{eq:mle}
\operatorname{argmax}_{\Pp\in \overline{\mathcal{P}}_{\mathcal{G}}(f)} \mathcal{L}(\Pp;\mathcal{D}) = \operatorname{argmax}_{\Pp\in \overline{\mathcal{P}}_{\mathcal{G}}(f)}\prod_{(\bx,c) \in \mathcal{D}} \Pp(\bX=\bx, C=c).
\end{equation}

The iterative proportional fitting (IPF) algorithm \citep{fienberg70II,lauritzen1996} can be used to solve Equation~\eqref{eq:mle}. 
Let $\mathcal{C}({\mathcal{G}})$ be the set of cliques of the graph $\mathcal{G}$ and $\Pp \in \mathcal{P}_{\mathcal{G}}\cap \Phi^{-1}(f)$. For $A \in \mathcal{C}({\mathcal{G}})$ define the marginal fitting operator:
$$ T_{A}\Pp(\bX=\bx, C=c) = \Pp(\bX=\bx, C=c) \frac{N(\bx_A)/N}{\Pp(\bX_A=\bx_A)},$$
where $N$ is the size of the dataset $\mathcal{D}$ and $N(\bx_A)$ is, for every $\bx_A \in \bchi_A$, the number of observations in the dataset $(\bx',c) \in \mathcal{D}$ such that $\bx'_A = \bx_A$.
Observe that,
\begin{equation} \label{eq:ipfFun}
 f_{T_A \Pp}= \ln\left(\frac{T_A\Pp(\bX=\bx,C=+1)}{T_A\Pp(\bX=\bx,C=-1)}\right)=
\ln\left(\frac{\Pp(\bX=\bx,C=+1)}{\Pp(\bX=\bx,C=-1)}\right)=f_{\Pp}.
\end{equation}
Thus $T_A\Pp \in \mathcal{P}_{\mathcal{G}}\cap \Phi^{-1}(f)$.
Given an ordering of the cliques $\mathcal{C}({\mathcal{G}})$, the IPF algorithm iteratively adjusts the marginal distribution of the cliques until convergence \citep{fienberg70II}. 
 
If we initialize the IPF algorithm with a probability in $\mathcal{P}_{\mathcal{G}}\cap \Phi^{-1}(f)$, for example
$$ \Pp^{0}(\bX=\bx,C=c) \varpropto \exp{\left(\frac{c}{2}f(\bx)\right)}, $$ 
then, by Equation~\eqref{eq:ipfFun}, the resulting maximum-likelihood estimation obtained with the IPF algorithm will be an element of $\overline{\mathcal{P}}_{\mathcal{G}}(f)$.

\subsection{Combining the Discriminative and Generative Approaches}

It is obvious now how to combine the discriminative and generative approaches.
Suppose to have a dataset $\mathcal{D}$ and a discriminative learning algorithm
that outputs an estimated function $\hat{f}_{\mathcal{D}} \in \mathcal{F}$.
Moreover, assume that, for a given graph $\mathcal{G}$,
$\Delta^{2}_{\mathcal{G}} \hat{f}_{\mathcal{D}} \equiv 0$. We can find, using
the IPF algorithm as described in Section~\ref{sec:mle}, the maximum likelihood $\mathcal{G}$-Markov classifier that induces $\hat{f}_{\mathcal{D}}$.

For various discriminative algorithms it is possible to know a priori the graph $\mathcal{G}$ such that $\Delta^{2}_{\mathcal{G}} \hat{f}_{\mathcal{D}} \equiv 0$. That is because a decomposition of $\hat{f}_{\mathcal{D}}$ is known as
$$ \hat{f}_{\mathcal{D}}(\bx) = \sum_{A} g_{A}(\bx_A) .$$

Examples include logistic regression~\citep{Cox58}, support vector
machines~\citep{Cortes1995} and Boolean classifiers (e.g using monomials basis)~\citep{odonnell2014}. 
Recently \cite{Zaidi2017} proposed a method to learn in a discriminative way
parameters of Bayesian network classifiers thus obtaining an estimated
discrimination function $\hat{f} \in \mathcal{F}_{\mathcal{G}}$ from which the
fixed discrimination maximum-likelihood generative classifier in
$\overline{\mathcal{P}}_{\mathcal{G}}(\hat{f})$ could be obtained with the method described in the previous section.

\section{Conclusions and Future Work}
\label{sec:conclusions}

In this paper we analyzed the impact of conditional independence statements and
in general of the undirected Markov property over the induced discrimination
function of a generative classifier. For this, we used a categorical
differential operator ($\Delta_A$) and we showed that conditional
independence statements are described by second-order equations in this operator
($\Delta_A\Delta_Bf \equiv 0$).  We then studied the implications for the
corresponding decision functions, that is, the binary classifier induced by
$\mathcal{G}$-Markov classifiers. Summarizing, we showed that a
$\mathcal{G}$-Markov classifier can represent XOR functions only among its
cliques.

We think that the given descriptions of conditional independence statements for
discrimination functions could be useful to study generative classifiers over
categorical predictors and to help design new type of learning procedures.  In
particular for future works we are interested in empirical error minimization
parameter estimation for $\mathcal{G}$-Markov classifiers and thus implementing
the mixed discriminative-generative approach described in
Section~\ref{sec:learning}.

\section*{Acknowledgments}

We thank Professor Henry Wynn for suggesting the use of finite differences.    

E. Riccomagno has been partially supported by the GNAMPA-INdAM 2017 project.
G. Varando, C. Bielza and P. Larrañaga by the Spanish Ministry of Economy and
Competitiveness through the Cajal Blue Brain (C080020-09; the Spanish partner of
the Blue Brain initiative from EPFL) and TIN2016-79684-P projects, by the
Regional Government of Madrid through the S2013/ICE-2845-CASI-CAM-CM project,
and by Fundación BBVA grants to Scientific Research Teams in Big Data 2016. 
G. Varando  by the Italian Association for Artificial Intelligence through
the Incoming Mobility Grant and by the Universidad Polit\'ecnica de Madrid through
the Programa Propio de I+D+I 2017.

\bibliography{biblio}

\begin{thebibliography}{28}
\providecommand{\natexlab}[1]{#1}
\providecommand{\url}[1]{\texttt{#1}}
\expandafter\ifx\csname urlstyle\endcsname\relax
  \providecommand{\doi}[1]{doi: #1}\else
  \providecommand{\doi}{doi: \begingroup \urlstyle{rm}\Url}\fi

\bibitem[Bielza and Larra\~{n}aga(2014)]{bielza2014}
Concha Bielza and Pedro Larra\~{n}aga.
\newblock Discrete {B}ayesian network classifiers: A survey.
\newblock \emph{ACM Computing Surveys}, 47\penalty0 (1):\penalty0 5:1--5:43,
  2014.

\bibitem[Cortes and Vapnik(1995)]{Cortes1995}
Corinna Cortes and Vladimir Vapnik.
\newblock Support-vector networks.
\newblock \emph{Machine Learning}, 20\penalty0 (3):\penalty0 273--297, 1995.

\bibitem[Cox(1958)]{Cox58}
D.~R. Cox.
\newblock The regression analysis of binary sequences.
\newblock \emph{Journal of the Royal Statistical Society. Series B
  (Methodological)}, 20\penalty0 (2):\penalty0 215--242, 1958.

\bibitem[Devroye et~al.(1996)Devroye, Gy{\"{o}}rfi, and Lugosi]{ptpr}
Luc Devroye, Laszlo Gy{\"{o}}rfi, and Gabor Lugosi.
\newblock \emph{A Probabilistic Theory of Pattern Recognition}.
\newblock Springer, 1996.

\bibitem[Diaconis and Sturmfels(1995)]{Diaconis95}
Persi Diaconis and Bernd Sturmfels.
\newblock Algebraic algorithms for sampling from conditional distributions.
\newblock \emph{Annals of Statistics}, 26:\penalty0 363--397, 1995.

\bibitem[Domingos and Pazzani(1997)]{Domingos97}
Pedro Domingos and Michael Pazzani.
\newblock On the optimality of the simple {B}ayesian classifier under zero-one
  loss.
\newblock \emph{Machine Learning}, 29\penalty0 (2-3):\penalty0 103--130, 1997.

\bibitem[Drton et~al.(2009)Drton, Sturmfels, and Sullivant]{drton2009lectures}
Mathias Drton, Bernd Sturmfels, and Seth Sullivant.
\newblock \emph{Lectures on Algebraic Statistics}.
\newblock Oberwolfach Seminars. Birkh{\"a}user Basel, 2009.

\bibitem[Fienberg(1970)]{fienberg70II}
Stephen~E. Fienberg.
\newblock An iterative procedure for estimation in contingency tables.
\newblock \emph{The Annals of Mathematical Statistics}, 41\penalty0
  (3):\penalty0 907--917, 1970.

\bibitem[Fienberg and Gilbert(1970)]{fienberg70}
Stephen~E. Fienberg and John~P. Gilbert.
\newblock The geometry of a two by two contingency table.
\newblock \emph{Journal of the American Statistical Association}, 65\penalty0
  (330):\penalty0 694--701, 1970.

\bibitem[Gandolfi and Lenarda(2017)]{gandolfi2017}
Alberto Gandolfi and Pietro Lenarda.
\newblock A note on {G}ibbs and {M}arkov random fields with constraints and
  their moments.
\newblock \emph{Mathematics and Mechanics of Complex Systems}, 4\penalty0
  (3-4):\penalty0 407–422, 2017.

\bibitem[Garcia et~al.(2005)Garcia, Stillman, and Sturmfels]{GARCIA2005}
Luis~D. Garcia, Michael Stillman, and Bernd Sturmfels.
\newblock Algebraic geometry of {Ba}yesian networks.
\newblock \emph{Journal of Symbolic Computation}, 39\penalty0 (3):\penalty0 331
  -- 355, 2005.

\bibitem[Grimmett(1973)]{grimmet1973}
Geoffrey~R. Grimmett.
\newblock A theorem about random fields.
\newblock \emph{Bulletin of the London Mathematical Society}, 5\penalty0
  (1):\penalty0 81--84, 1973.

\bibitem[Hammersley and Clifford(1971)]{hammersley1971}
John Hammersley and Peter Clifford.
\newblock Markov fields on finite graphs and lattices.
\newblock 1971.
\newblock Unpublished manuscript.

\bibitem[Jaeger(2003)]{Jaeger03}
Manfred Jaeger.
\newblock Probabilistic classifiers and the concepts they recognize.
\newblock In \emph{Proceedings of the Twentieth {I}nternational Conference on
  Machine Learning}, pages 266--273. AAAI Press, 2003.

\bibitem[Lauritzen(1996)]{lauritzen1996}
Steffen~L. Lauritzen.
\newblock \emph{Graphical Models}.
\newblock Clarendon Press, 1996.

\bibitem[Li and Yang(2018)]{LI2018}
Benchong Li and Youlong Yang.
\newblock Complexity of concept classes induced by discrete {M}arkov networks
  and {B}ayesian networks.
\newblock \emph{Pattern Recognition}, 2018.

\bibitem[Ling and Zhang(2002)]{ling2003}
Charles~X. Ling and Huajie Zhang.
\newblock The representational power of discrete {B}ayesian networks.
\newblock \emph{Journal of Machine Learning Research}, 3:\penalty0 709--721,
  2002.

\bibitem[Minsky(1961)]{Minsky61stepstoward}
Marvin Minsky.
\newblock Steps toward artificial intelligence.
\newblock In \emph{Computers and Thought}, pages 406--450. McGraw-Hill, 1961.

\bibitem[Ng and Jordan(2001)]{Ng2001}
Andrew~Y. Ng and Michael~I. Jordan.
\newblock On discriminative vs. generative classifiers: A comparison of
  logistic regression and naive {B}ayes.
\newblock In \emph{Proceedings of the 14th International Conference on Neural
  Information Processing Systems: Natural and Synthetic}, NIPS'01, pages
  841--848, Cambridge, MA, USA, 2001. MIT Press.

\bibitem[O'Donnell(2014)]{odonnell2014}
Ryan O'Donnell.
\newblock \emph{Analysis of {B}oolean Functions}.
\newblock Cambridge University Press, 2014.

\bibitem[Pearl(1988)]{Pearl1988}
Judea Pearl.
\newblock \emph{Probabilistic Reasoning in Intelligent Systems: {N}etworks of
  Plausible Inference}.
\newblock Morgan Kaufmann Publishers Inc., 1988.

\bibitem[Peot(1996)]{peot}
Mark~A. Peot.
\newblock Geometric implications of the naive {B}ayes assumption.
\newblock In \emph{Proceedings of the Twelfth International Conference on
  Uncertainty in Artificial Intelligence}, pages 414--419. Morgan Kaufmann
  Publishers Inc., 1996.

\bibitem[Rota(1987)]{Rota1987}
Gian-Carlo Rota.
\newblock On the foundations of combinatorial theory.
\newblock In Ira Gessel and Gian-Carlo Rota, editors, \emph{Classic Papers in
  Combinatorics}, pages 332--360. Birkh{\"a}user Boston, 1987.

\bibitem[Settimi and Smith(1998)]{Settimi1998}
Raffaella Settimi and Jim~Q. Smith.
\newblock On the geometry of {B}ayesian graphical models with hidden variables.
\newblock In \emph{Proceedings of the Fourteenth Conference on Uncertainty in
  Artificial Intelligence}, pages 472--479. Morgan Kaufmann Publishers Inc.,
  1998.

\bibitem[Vapnik(2000)]{vapnik}
Vladimir Vapnik.
\newblock \emph{The Nature of Statistical Learning Theory}.
\newblock Springer, 2000.

\bibitem[Varando et~al.(2015)Varando, Bielza, and Larra{\~n}aga]{varando15}
Gherardo Varando, Concha Bielza, and Pedro Larra{\~n}aga.
\newblock Decision boundary for discrete {B}ayesian network classifiers.
\newblock \emph{Journal of Machine Learning Research}, 16:\penalty0 2725--2749,
  2015.

\bibitem[Webb et~al.(2005)Webb, Boughton, and Wang]{Webb2005}
Geoffrey~I. Webb, Janice~R. Boughton, and Zhihai Wang.
\newblock Not so naive {B}ayes: Aggregating one-dependence estimators.
\newblock \emph{Machine Learning}, 58\penalty0 (1):\penalty0 5--24, Jan 2005.

\bibitem[Zaidi et~al.(2017)Zaidi, Webb, Carman, Petitjean, Buntine, Hynes, and
  De~Sterck]{Zaidi2017}
Nayyar~A. Zaidi, Geoffrey~I. Webb, Mark~J. Carman, Fran{\c{c}}ois Petitjean,
  Wray Buntine, Mike Hynes, and Hans De~Sterck.
\newblock Efficient parameter learning of {B}ayesian network classifiers.
\newblock \emph{Machine Learning}, 106\penalty0 (9):\penalty0 1289--1329, 2017.

\end{thebibliography}

\end{document}